\documentclass[letterpaper, 10 pt, conference]{template/ieeeconf}  

\IEEEoverridecommandlockouts                              
\usepackage{graphics} 
\usepackage{graphicx}
\usepackage[tight,footnotesize]{subfigure}
\usepackage{caption}
\usepackage{epsfig} 
\usepackage{mathptmx,cuted} 
\usepackage{times} 
\usepackage{amsmath} 
\usepackage{amssymb}  
\usepackage{hyphenat}
\usepackage[linesnumbered,ruled]{algorithm2e}
\usepackage{color}

\usepackage[utf8]{inputenc}
\usepackage[english]{babel}

\newtheorem{theorem}{Theorem}

\newtheorem{example}{Example}
\usepackage[tight,footnotesize]{subfigure}
\long\def\invis#1{}
\newcommand{\C}{\mathcal{C}}
\newcommand{\Cfree}{\C_\textrm{free}}
\newcommand{\Cobs}{\C_\textrm{obs}}
\newcommand{\Csub}{\C_\textrm{sub}}
\newcommand{\qinit}{q_{\rm init}}
\newcommand{\qgoal}{q_{\rm goal}}
\newcommand{\R}{\mathbb{R}}

\def\floatsep{4.0pt}
\def\textfloatsep{0.0pt}
\def\dblfloatsep{4.0pt}
\def\dbltextfloatsep{4.0pt}
\title{\LARGE \bf RRT$^+$: Fast Planning for High-Dimensional Configuration Spaces}

\author{Marios Xanthidis$^{1}$, Ioannis Rekleitis$^{1}$, Jason M. O'Kane$^{1}$
\thanks{$^{1}$Marios Xanthidis, Ioannis Rekleitis and Jason M. O'Kane are with the Computer Science and Engineering Department, University of South Carolina,
United States, {\tt\small mariosx@email.sc.edu, [yiannisr, jokane]@cse.sc.edu}}%
}
\begin{document}

\maketitle
\thispagestyle{empty}
\pagestyle{empty}

\begin{abstract}
In this paper we propose a new family of RRT based algorithms, named RRT$^+$, that are able to find faster solutions in high-dimensional configuration spaces compared to other existing RRT variants by finding paths in lower dimensional subspaces of the configuration space. The method can be easily applied to complex hyper-redundant systems and can be adapted by other RRT based planners. We introduce RRT$^+$ and develop some variants, called PrioritizedRRT$^+$, PrioritizedRRT$^+$-Connect, and PrioritizedBidirectionalT-RRT$^+$, that use the new sampling technique and we show that our method provides faster results than the corresponding original algorithms.  Experiments using the state-of-the-art planners available in OMPL show superior performance of RRT$^+$ for high-dimensional motion planning problems.   \invis{Moreover, the algorithm can be used for producing paths that minimize the motion on certain DOFs, if possible, in case the system has certain known constraints for specific tasks.}

\end{abstract}


\section{Introduction} 
\label{sec_intro}

Despite the dramatic development of robotics in recent decades, robots are still far from outperforming humans in operations for which they are not specialized, however few types of robots are capable of completing generic tasks. \invis{At the same time, only a few robots can execute more general tasks.} Complex robots such as mobile manipulators, snake robots, or humanoids have been presented mostly in experimental contexts. While the fact that a human adult has 244~degrees of freedom~\cite{kuo1994mechanical}, emphasizes the need for such complicated systems, which would compare favorably with human performance in many domains.

\invis{The reasons for not having already powerful complex robotic systems, considering the different fields of robotics, can be partitioned into four big issues: The structural complexity that arises from complicated systems, the observability of the system that becomes harder with more complicated structures, the unpredictable events in the environment that can cause unpredictable behavior in more complicated robots, and real-time planning that becomes more timely expensive, even affordable, in higher dimensions.}

This study focus on the path planning problem in high-dimensional configuration spaces and aims to provide a new family of motion planners that generate faster solutions than other RRT-based motion planners.\invis{, without a near-optimality guarantee.}  The approach is based on the observation that
for many hyper-redundant systems, it is rare for all the kinematic abilities of the system to be needed for a certain task~\cite{RekleitisMed2016}.
\invis{Therefore, the proposed algorithms utilize constraints on the DOFs to find solutions faster.}

\begin{figure}[ht]
\centering
\fbox{\includegraphics[width=0.4\textwidth, clip=true, trim=0.0in 3.5in 2.2in 1.5in]{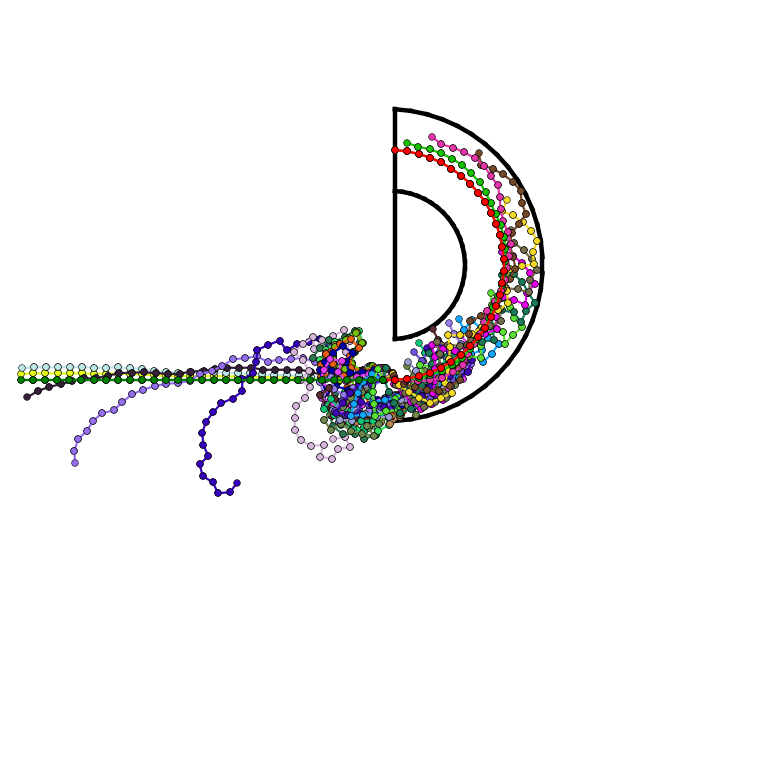}}
\caption{Solution for a 30-DOF kinematic chain in a highly constrained horn environment using the proposed algorithm.}
\label{fig:4d}
\end{figure}

\invis{\begin{itemize}
\item In many critical cases, having fast and non optimal solutions can guarantee the survival of the system, rather than having belated optimal, or near optimal, ones.
\end{itemize}}

In this paper, a novel RRT-based family of algorithms, termed
RRT$^+$, is presented.  These algorithms incrementally search in subspaces of the configuration space
$\C$. The modifications are not only related to the sampler but also to the way
the RRT is expanded since the entire search is separated into a series of
subsearches, confined to subspaces of increasing dimensionality.

\invis{The subsearches are related to each different
subspace where the tree is expanded, while when a sub-search is finished, the
sampler is refined for the next subsearch.} We show that these algorithms keep
all the theoretical guarantees of the original RRT algorithms, while at the same time can provide
much faster solutions in many cases.
\invis{although in theory in the uncommon worst case can have a
sufficiently slower performance.}
Moreover, the idea can easily be grafted to other motion planners --- in the
authors' experience, sometimes in just a few minutes of programming --- resulting in
significant increases in performance.
Furthermore, the efficiency can be significantly increased by having prior
information about the system. 
Figure \ref{fig:4d} presents an example plan for a 30 DoF manipulator in a tight environment computed by one of the proposed algorithms, namely PrioritizedRRT$^+$-Connect.

\invis{Last the PrioritizedRRT$^+$ is presented as an easy to apply planner in the RRT$^+$ family, that can be tuned to provide much faster results and also has the ability to find paths with a minimum range for some desired DOFs, if possible.}

The remainder of this paper is structured as follows, related work and other
approaches to these problems is discussed in Section~\ref{sec:background}. The
problem statement is in Section~\ref{sec:Problem} and the proposed algorithm is
presented in Section~\ref{sec:Algorithm}. Section~\ref{Experim} presents
experiments comparing RRT$^+$ to five other state-of-the-art planners from
OMPL.  Finally, the paper concludes with a discussion of future work in
Section~\ref{sec:Conclusion}.

\section{Related work}
\label{sec:background}
The problem of the motion planning has been proven to be PSPACE-hard~\cite{reif1994motion}.\invis{This means that true completeness is a very rare guarantee while at the same time the algorithms that provide near optimal solutions have significant computational costs.} Building robotic systems with high-dimensional configuration spaces is a field that many studies have been done for different type of systems. Generally there are four big categories: hyper-redundant manipulators, mobile manipulators, cooperative robotic systems and humanoids. We are going to present some of the important studies on those systems. \invis{before and after the introduction of sampling-based methods.}

There are some works that show applications of hyper-redundant systems as demonstrated by Ma \emph{et al.}~\cite{ma1994development} who developed a hyper-redundant arm in order to make real time and precise operations inside nuclear reactors. Liljeback~\emph{et al.}~\cite{liljeback2006snakefighter} created a snake fire-fighting hyper-redundant robot, and Ikuta \emph{et al.}~\cite{ikuta2003hyper} used a 9-DoF arm to operate during surgeries in deep areas.

There are also few works for mobile manipulators and multi-robot systems from the early 90's. The inverse kinematics for mobile manipulators where solved and optimal solutions were obtained with guarantees of safe performance in the study of Dubowsky and Vance~\cite{dubowsky1989planning}. Later studies were presented, such as the one of Yamamoto and Yun~\cite{yamamoto1992coordinating} where an algorithm for maximizing the manipulability was presented, and the work of Khabit \emph{et al.}~\cite{khatib1996vehicle} introduced an algorithm for cooperative systems of mobile manipulators. As for the cooperative systems, some early efforts include the works of Buckley~\cite{buckley1989fast} and Cao \emph{et al.}~\cite{cao1997cooperative}.

From late 90's, sampling-based methods were introduced and shown to be capable of solving challenging motion planning problems, but without guaranties for finding the solution in finite time~\cite{lavalle2006planning}. The two most famous representatives of those algorithms are probabilistic roadmaps (PRMs) by Kavraki \emph{et al.} that require preprocessing and a known, generally stable, environment~\cite{kavraki1996probabilistic} and RRTs by LaValle~\cite{lavalle1998rapidly}, that are more suitable for single query applications. Although the performance of both these techniques can be affected substantially by the number of the degrees of freedom, some studies have use them successfully for high-dimensional configuration spaces.
 
A method that uses PRM was presented by Park \emph{et al.}~\cite{park2011collision} that finds collision-free paths for hyper-redundant arms. Other studies use RRTs for motion planning of redundant manipulators, such as the work of Bertram \emph{et al.}~\cite{bertram2006integrated}, which solves the inverse kinematics in a novel way.  Weghe \emph{et al.}~\cite{weghe2007randomized} apply RRT to redundant manipulators without the need to solve the inverse kinematics of the system. A study by Qian and Rahmani~\cite{qian2013path} combines the RRT and the inverse kinematics in a hybrid algorithm in a way that drives the expansion of the RRT by the Jacobian pseudo-inverse. 

Additionally, some works use RRTs for mobile manipulators. Vannoy \emph{et al.}~\cite{vannoy2008real} proposes an efficient and flexible algorithm for operating in dynamic environments.  The work of Berenson \emph{et al.}~\cite{berenson2008optimization} provides an application of their technique to a 10-DoF mobile manipulator.

For multi-robot systems, many sampling-based algorithms have been proposed.  The study of van den Berg and Overmars~\cite{van2005prioritized} uses a PRM and presents a prioritized technique for motion planning of multiple robots. Other studies use RRT-based algorithms such as the study by Carpin and Pagello~\cite{carpin2002parallel} which introduced the idea of having multiple parallel RRTs for multi-robot systems, or the work of Wagner~\cite{wagner2015subdimensional} that plans for every robot individually and the it tries to coorbinate the motion if needed in higher dimensional spaces.  Other studies propose efficient solutions by using a single RRT~\cite{otani2009applying, solovey2015finding}. There is also some work on humanoid robots with sampling-based algorithms; Kuffner \emph{et al.} presented algorithms for motion planning on humanoid robots with both the use of PRMs~\cite{kuffner2002dynamically} and RRTs~\cite{kuffner2005motion}.  Other studies, such as the work of Liu \emph{et al.}~\cite{liu2012hierarchical}, which used RTTs for solving the stepping problem for humanoid robots. 

\invis{But while there are many different efficient algorithms for different purposes based on sampling, as it has been already described, there is very little work on proposing a general method for dealing with big configuration spaces.}  The work of Vernaza and Lee tried to extract structural symmetries in order to reduce the dimensionality of the problem~\cite{vernaza2011efficient} by also providing near-optimal solutions. This technique is more time efficient than the traditional RRT only for very high-dimensional configuration spaces. \invis{and only for known environments where the cost function is more or less stable.} Yershova \emph{et al.}~\cite{yershova2005dynamic} proposed an approach to focus the sampling in the most relevant regions. 

Finally, some other approaches attempt to deal with the curse of dimensionality in different ways. Gipson \emph{et
al.}~developed STRIDE\cite{gipson2013resolution} which samples non-uniformly
with a bias to unexplored areas of $\C$. \invis{where the density of samples is lower,
provided a fast planner for high-dimensional configuration spaces.}
Additionally, Gochev \emph{et al.}~\cite{gochev2011path} proposed a motion
planner that decreases the dimensionality by recreating a configuration space
with locally adaptive dimensionality. Yoshida's work~\cite{yoshida2005humanoid}
tries to sample in ways that exploit the redundancy of a humanoid system.  \invis{by modeling the system
with different ways that each one forms a subspace in the configuration space
and then with the use of a specific sample generator for every different
configuration model and the adaptive changes of the configurations of the
system the planner provides fast solutions.} Kim \emph{et al.}~\cite{kim2015efficient} present an RRT-based algorithm for articulated robots
that reduces the dimensionality of the problem by projecting each sample into
subspaces that are defined by a metric. Shkolnik and Tedrake ~\cite{shkolnik2009path} plan for high redundant manipulators
in the a dimensional task space with the use of Jacobian transpose and Vornoi bias. Last it is worth to mention
the work of Bayazit\emph{et al.}~\cite{bayazit2005iterative} where a PRM was used to plan in subspaces of $\C$,
creating paths that solve an easier problem than the original and then iteratively optimize the solution by extending the
subspace till it finds a valid solution. 

This paper provides a general and simple to implement RRT-based family of algorithms for efficient motion planning for arbitrary hyper-redundant systems, regardless if there is an articulated robot, or a humanoid, or an heterogeneous multi-robot system. unlike the previous works, the algorithm tries to find paths that are not only confined entirely in a global subspace but also in a subspace in which a solution can exist, since the initial and goal configurations are part of the same subspace. The algorithm not only exploits redundancy but also can provide fast solutions that satisfy some constraints by simply searching in subspaces where those constraints are satisfied.  Contrary to ~\cite{bayazit2005iterative} the algorithm searches in subspaces that are strictly lower-dimensional and instead of solving an easier problem than the original, since refining the tree will be costly for an RRT, it tries to solve a much more difficult problem by applying virtual constraints, and iteratively relaxes those virtual constraints.\invis{so the tree tries to expand only in a way that satisfies these constraints.} The algorithm is easily applicable to a variety RRT-based motion planners; we demonstrate this property by applying one variant of our method to RRT~\cite{lavalle1998rapidly}, RRTConnect~\cite{kuffner2000rrt} and Bidirectional T-RRT~\cite{jaillet2008transition}.

\section{Problem Statement}
\label{sec:Problem}

Let $\C$ denote a configuration space with $n$ degrees of freedom, partitioned
into free space $\Cfree$ and obstacle space $\Cobs$ with $\C= \Cfree \cup
\Cobs$.
The obstacle space $\Cobs$ is not explicitly represented, but instead can be
queried using collision checks on single configurations or short path segments.
Given initial and goal configurations $\qinit, \qgoal \in \Cfree$, we
would like to find a continuous path in within $\mathcal{C}_{\rm free}$ from
$\qinit$ to $\qgoal$.

\newcommand{\ssmin}{^{\rm (min)}}  
\newcommand{\ssmax}{^{\rm (max)}}  
\newcommand{\cint}[1]{\left[c_{#1}\ssmin, c_{#1}\ssmax\right]}  

For purposes of sampling, we assume that each degree of freedom in $\C$ is
parameterized as an interval subset of $\R$, 
so that
  \begin{equation}
    \C = \cint{1} \times\dots\times \cint{n} \subseteq \R^n. \label{eq:C}
  \end{equation}
Note that we treat $\C$ as Euclidean only in the context of sampling; other
operations such as distance calculations and generation of local path segments
utilize identifications on the boundary of $\C$ as appropriate for the
topology.

\section{Algorithm description}
\label{sec:Algorithm}
\subsection{The RRT$^+$ Algorithm}
The proposed family of algorithms is based on attempting to find a solution in lower
dimensional subspace of $\C$, in hopes that such a path might be found faster
than expanding the tree in all dimensions. 
The underlying idea is to exploit the redundancy of each system for each
problem.  To achieve this, the algorithm starts searching in a 1-dimensional
subspace of $\C$ that contains $\qinit$ and $\qgoal$.  If this search fails, the algorithm expands its search
subspace by one dimension.  This process continues iteratively until the
algorithm finds a path, or until it searches in all of $\C$. During all the
different searches, also called sub-searches, the tree structure is kept and
expanded in subsequent stages.

\begin{algorithm}[t]
  \caption{RRT$^+$}\label{RRT+}
  \SetKwInOut{Input}{Input}
  \SetKwInOut{Output}{Output}
  \DontPrintSemicolon
  \SetAlgoLined
  \Input{A configuration space $\C$, an initial configuration $\qinit$, and a goal configuration $\qgoal$.}
  \Output{RRT graph $G$}

  $G$.init($q_{init}$)\;
  $\Csub \gets $ 1-d subspace of $\C$, through $q_{init}$ and $q_{goal}$\; \label{lineA}

  \While{True}{
    $q_{\rm rand} \gets$ sample drawn from $\Csub$ \; \label{lineC}
    $q_{\rm near} \gets $NearestVertex$(q_{\rm rand},G)$\;
    $q_{\rm new} \gets $NewConf$(q_{\rm near}, q_{\rm rand})$\;
    $G$.AddVertex($q_{\rm new}$)\;
    $G$.AddEdge($q_{\rm near}$, $q_{\rm new}$)\;
    \If{done searching $\Csub$\label{lineD}}{
      \uIf{$\dim(\Csub) < \dim(\C)$}{
        Expand $\Csub$ by one dimension. \label{lineB}
      }
      \Else{
        \Return $G$
      }
    }
  }
\end{algorithm}

Algorithm~\ref{RRT+} summarizes the approach.
The planner starts optimistically by searching in one dimension, along the line
passing through $\qinit$ and $\qgoal$.  If this search fails to find a path---a
certainty, unless there are no obstacles between $\qinit$ and $\qgoal$---the
search expands to a planar subspace that includes $\qinit$ and $\qgoal$, then
to a 3D hyperplane, and so on until, in the worst case, the algorithm
eventually searches all of $\C$; see Fig.~\ref{fig:samples}.

For simplicity, the pseudocode in Algorithm~\ref{RRT+} shows a straightforward,
single-directional RRT$^+$, analogous to the standard vanilla RRT.
Note, however, the idea readily applies to most tree-based motion planners,
since primary difference is in how the samples are generated.
The experiments in Section~\ref{Experim}, for example, describe RRT$^+$ planners
based on the well-known goal biased and bidirectional RRT algorithms.

\begin{figure*}[t]
 \begin{center}
  \leavevmode
   \begin{tabular}{ccc}
     \subfigure[]{\includegraphics[width=0.25\textwidth]{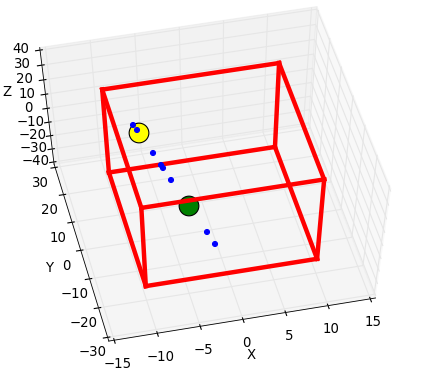}}&
     \subfigure[]{\includegraphics[width=0.25\textwidth]{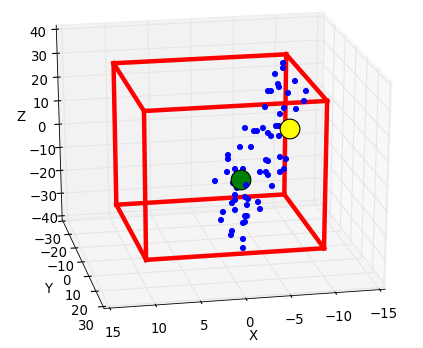}}&
     \subfigure[]{\includegraphics[width=0.25\textwidth]{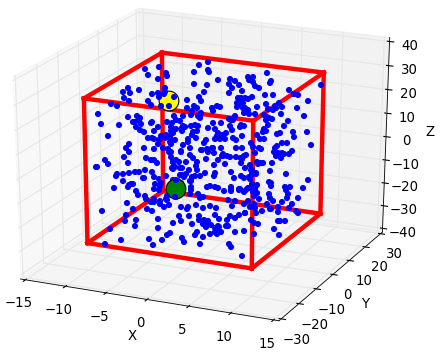}}
    \end{tabular}
  \end{center}
\caption{Sampling in one, two and three dimensions in $\C$.}
\label{fig:samples}
\end{figure*}

The description in Algorithm~\ref{RRT+} leaves three important elements
unspecified.
First, the algorithm needs a method for selecting and representing the subspace
$\Csub$ (Lines~\ref{lineA} and \ref{lineB}).
Next, we must provide a method for sampling from this subspace (Line~\ref{lineC}).
Finally, we must decide upon the conditions that must be met before moving to
the next subsearch (Line~\ref{lineD}).
Particular choices for each of these elements can be based on knowledge about
the specific types of problems being solved; this freedom is why we refer to
RRT$^+$ as a family of algorithms.  Sections~\ref{sec:subspace},
\ref{sec:sampling}, and \ref{sec:termination} describe some possibilities for
instantiating the basic framework.  Section~\ref{sec:pc} establishes conditions
under which the resulting planner is probabilistically complete.


\subsection{Choosing and representing subspaces}\label{sec:subspace}

The central idea is to search for solutions in subspaces of progressively
higher dimensions.  The primary constraint on these subspaces is that they must
contain both $\qinit$ and $\qgoal$; subspaces that violate this constraint
cannot, of course, contain a path connecting $\qinit$ to $\qgoal$.

In general, the algorithm's selections for $\Csub$ should ideally be directed
by the likelihood that a solution will exist fully within $\Csub$, but in the
absence of useful heuristics for predicting this success, simple randomness may
be quite effective, especially for highly-redundant systems.
The next examples illustrate some options.

\begin{example}\label{ex1}
  A natural choice is to let $\Csub$ be an affine subspace of $\C$.  That is,
  we select at random a flat\footnote{We use the term \emph{flat} to refer to a
  subset of $\R^n$ congruent to some lower-dimensional Euclidean space.} in
  $\C$ passing through $\qinit$ and $\qgoal$.

  \invis{An $s$-dimensional hyperplane passing through $\qinit$ and $\qgoal$ can be
  defined as the set of points $q$ satisfying the equation
    $$ q^\top = Ax^\top + \qinit^\top. $$
  in which ...$x$ is a vector $s \times 1$ \textbf{TODO: This equation is not correct, because both $q$ and
  $x$ are unknowns. Replace this with an equation in which only $q$ is unknown,
  which is satisfied for points on the subspace we want.}

  Here $A$ is a full-rank $n \times s$ matrix of the following form
  \begin{equation}
    A = \begin{bmatrix}
      \qgoal^\top-\qinit^\top & L_1^\top-\qinit^\top & \cdots & L_{s-1}^\top-\qinit^\top
    \end{bmatrix},
  \end{equation}
  in which the vectors $L_1, \ldots, L_{s-1}$ are selected randomly.}

\end{example}

\begin{example}
  Another possibility is to use for $\Csub$ a convex polytope.  That is, we
  select a collection of points $q_1,\ldots,q_m \in \R^n$ whose convex hull
  contains both $\qinit$ and $\qgoal$, and define $\Csub$ as a the set of all
  convex combinations of those points.
\end{example}

\newcommand{\Pcon}{P_{\rm con}}

\begin{example}
  A final possibility ---one that trades some generality for simplicity--- is a
  prioritized release of the degrees of freedom.
  Given a set $\Pcon \subseteq \{1, \ldots, n \}$ of degrees of freedom to be
  constrained, we can form $\Csub$ by constaining the degrees of freedom in $\Pcon$ to 
  form a line passing from $\qinit$ and $\qgoal$ and allowing the remaining DoFs to vary freely.

\end{example}

\subsection{Sampling from the subspaces}\label{sec:sampling}
Next, the algorithm requires a technique for drawing samples from $\Csub$.
\invis{Since three different methods were described briefly for choosing hyperplanes, ways to sample from the hyperplanes. Each different method and each own advantages or disadvantages will be discussed below:}
\begin{example}
  For affine subspaces represented as in Example~\ref{ex1}, samples in $\Csub$
  can be generated by selecting a random vector $r \in \R^s$,\invis{\textbf{(TODO:
  Need to say something, either here or in the enumerate below, about what
  portion of $\R^n$ the $r$ vectors are sampled from.)}} and applying an
  affine transformation:
  \begin{equation}\label{matrix}
    q_{\rm rand}= Ar + \qinit.
  \end{equation}
  The main difficulty is to ensure that the resulting $q$ lies inside the
  C-space as defined in Equation~\ref{eq:C}.  Possibilities for handling this
  difficult include rejection sampling, computing the portion of $\R^s$ that
  maps into $\C$ and generating $r$ from that region, or simply allowing
  samples to fall outside of $\Csub$, and tolerating the potential distortion
  this would induce to the growth of the tree.

\end{example}

\begin{example}
With the formulation described in
Example~2, uniform samples from the convex hull of $q_1,\ldots,q_m$ can be
generated using the algorithms of Linial~\cite{linial1986hard} or
Trikalinos and Valkenhoef~\cite{trikalinos2014efficient}.

\end{example}

\begin{example}
Prioritized sampling as described in Example~3, a very efficient linear time
method to produce samples \invis{in the intersection of hyperplanes passing
through $\qinit$ and $\qgoal$} is by initially generating a sample within $\C$
along the line between $\qinit$ and $\qgoal$ by selecting a random scalar $r$
and applying:
  \begin {equation} \label{get_sample}
    q_{\rm rand}^{(i)}= (\qgoal^{(i)}-\qinit^{(i)})r+\qinit^{(i)}
  \end{equation}
The algorithm then modifies $q_{\rm rand}$ by inserting, for each DoF not in
$\Pcon$, a different random value within the range for that dimension; see Algorithm~\ref{sampler}.

\begin{algorithm}
\caption{Prioritized Sampler}
\label{sampler}
\SetKwInOut{Input}{Input}
\SetKwInOut{Output}{Output}
\DontPrintSemicolon
\SetAlgoLined
\Input{Initial configuration $\qinit$, goal configuration $\qgoal$, set of constrained DoF $\Pcon$}
\Output{Sample configuration $q$}

$q \gets $ random point in $\C$ along line from $\qinit$ to $\qgoal$ \;
\For{$s \in 1,\ldots,n$}{
	\If{$s \notin \Pcon$}{
	  $q[s] \gets$ Random($0$, $1$)$*(c^{(max)}_s-c^{(min)}_s)+c^{(min)}_s$\;
	}
}

\Return $q$\;
\end{algorithm}

While this approach is simple, it loses some generality since only a finite number of flats can be explored; given specific $\qinit$, $\qgoal$, and $\Pcon$ the subspaces are fully determined.
\invis{By using the formal algebraic description of a hyperplane, every subspace can be formalized as:
\begin{equation}
a_1x_1+a_2x_2+...+a_ix_i+...+a_{n-1}x_{n-1} +a_nx_n = b
\end{equation}
So the subspace in every stage with dimensionality $Dim$ will be:
\begin{equation}
ax+a_1x_1+a_2x_2+...+a_ix_i+...+a_{Dim-1}x_{Dim-1} = b
\end{equation}

So the hyperplane in each iteration of the algorithm can be found with:
\begin{itemize}
\item $a=\sum D[i]$ and $ ratio_{min}\leq x\leq ration_{max}$, for every $i\in P_{con}$
\item $a_i=c^{(max)}_i-c^{(min)}_i$ and $0\leq x_i \leq 1$, for every $i\in P_{ind}$
\item $b=-(Sum_{con}+Sum_{free})$ where $Sum_{con}= \sum {q_{init} [i]}$ for every $i\in P_{con}$ and $Sum_{free}= \sum {C_{min} [i ]}$ for every $i \in P_{ind}$
\end{itemize}}
\end{example}

\subsection{Terminating the subsearches}\label{sec:termination}
The only remaining detail to 
be discussed is for how long the search in the each subspace should continue.
There are multiple ways to deal with this problem and good online
heuristics can be found, but for simplicity in this paper we assume that there
is a fixed number of samples $k_i$ for iteration $i$.  Algorithm~\ref{findsize}
presents a technique for selecting the $k_i$ values, based on a maximum number
of total samples $Q_{\rm total}$.  The idea is to exponentially increase the number
of samples in each successive subsearch, acknowledging the need for more
samples in higher dimensions.

\invis{For this purpose an additional vector $V=[v_1,v_2,\dots,v_n]$
containing the desired factors that the previous number of samples should be
multiplied to pick for each different hyper-plane should be introduced.
Starting with an initial number of samples for the first step in each $j$ step
the sampler samples $v_j*k_{j-1}$ number of samples.}

\begin{algorithm}
\caption{FindSampleSize\label{findsize}}
\SetKwInOut{Input}{Input}
\SetKwInOut{Output}{Output}
\DontPrintSemicolon
\SetAlgoLined
\Input{The total number of samples $Q_{\rm total}$, dimension of configuration space $n$.}
\Output{$k_1, \ldots, k_n$}

$v \gets \exp{\frac{\ln{Q_{\rm total}}}{n}}$\;
\For{$s \gets 1, \ldots, n$}{
	$k_s \gets v^s$\;
}
\Return $k_1,\ldots, k_n$\;
\end{algorithm}

For the sake of simplicity if we assume that $v_i=v$, for some constant $v$,
then the number of samples for each $j$ step will be:
\begin{equation}
  k_j=v^j
\end{equation} 

The total number of samples $Q_{\rm total}$ for all the steps will be the sum of the above geometric sequence:
\begin{equation}
Q_{\rm total}=\sum_{j=1}^n{v^j}=\frac{v(v^n-1)}{v-1}
\end{equation} 

\invis{By the these formulations the number of the samples for each state can be produced easily by just setting a desired maximum number of samples $Q_{max}$ as a time out, in the same way as in the regular RRTs. By simply find the proper $v$:

\begin{equation}
v=\exp{\frac{\ln{Q_{max}}}{n}}= {Q_{max}}^{\frac{1}{n}}
\end{equation} }

It can be shown that in the unlikely worst case the RRT$^+$ algorithm would be slightly slower than the regular RRT since:

\begin{equation}
Q_{max}\le \frac{v(Q_{max}-1)}{v-1} \approx \frac{v}{v-1}Q_{max} 
\end{equation}   

Also it is important that the efficiency of the algorithm in the worst case, is strongly related to the value of $v$. For bigger values of $v$, so bigger values of $Q_{\rm total}$, the difference in the efficiency in the worst case is reduced. So by setting as $Q_{max}$ the average number of samples that the algorithm before the application of the method the new algorithm succeeds on even in the worst uncommon case for a hyper-redundant system, having similar performance.

Last, putting everything together, the PrioritizedRRT$^+$, based on Examples~3
and 6, is presented in Algorithm \ref{PRRT+}.

\begin{algorithm}
\caption{Prioritized RRT$^+$\label{PRRT+}}
\SetKwInOut{Input}{Input}
\SetKwInOut{Output}{Output}
\DontPrintSemicolon
\SetAlgoLined
\Input{Initial configuration $\qinit$, goal configuration $\qgoal$, number of samples $Q_{\rm total}$.}
\Output{RRT graph $G$}

$P_{con} \gets$ $\{1,2,\dots,n\}$\;
$G$.init($q_{init}$)\;
$K \gets$ FindSampleSize($N$, $S$)\;
\For{$s \gets 1,\ldots, n$}{
	\For{\normalfont $i \gets 1, k_s$}{
	$q_{\rm rand} \gets $PrioritizedSampler($S$, $\qinit$, $\qgoal$, $\Pcon$)\;
	$q_{\rm near} \gets $NearestVertex$(q_{\rm rand},G)$\;
	$q_{\rm new} \gets $NewConf$(q_{\rm near}, q_{\rm rand})$\;
	$G$.AddVertex($q_{\rm new}$)\;
	$G$.AddEdge($q_{\rm near}$, $q_{\rm new}$)\;
}
Extract($\Pcon$)\;
}
\Return $G$\
\end{algorithm}

In Fig. \ref{fig:samples}, the visual result of the sampler for a 3-dimensional configuration space and the three different states, with $Q_{max}=512$.  The initial state is shown as a yellow ball; the goal state is shown as a green ball, and the red lines indicate the boundary of $\C$.

\subsection{Probabilistic Completeness}\label{sec:pc}

\begin{theorem}
  The RRT$^+$ is probabilistically complete if the following conditions hold.
  \begin{itemize}
    \item $\Csub$ in the last stage is the entire $\C$.
    \item The sampler can generate any point in $\Csub$.
    \item In each stage before the final stage, only a finite number of samples
    is generated.
  \end{itemize}
\end{theorem}
 
\begin{proof}
  Under these conditions, RRT$^+$ is guaranteed to reach its last iteration in
  finite time.  In this final iteration, the algorithm behaves in the exact
  same way as the RRT, so RRT$^+$ inherits the probabilistic completeness of
  RRT.
\end{proof}

\section{Experiments}
\label{Experim}

For the experiments three new planners in the RRT$^+$ family were developed in
within OMPL~\cite{sucan2012open}.  These planners work by applying the
prioritized technique to three RRT variants: (1) RRT with goal bias of $0.5$.
(2) RRT-Connect and (3) the bidirectional T-RRT~\cite{jaillet2008transition}.
The T-RRT variant is intended to demonstrate the applicability of the idea to a
powerful costmap planner.

In order to show the ability of the new planners to adapt in different problems, 
the prioritization was chosen randomly for each run, although an optimization 
might be expected to give faster results. For implementation reasons, the parameter $Q_{max}$ was specified
not by indicating the number of the samples, but the 
total desired time the subsearch should be done to a machine with 6th Generation 
Intel Core i7-6500U Processor (4MB Cache, up to 3.10 GHz) and 16GB of DDR3L 
(1600MHz) RAM.  As expected the performance proved to be sensitive to this parameter.
 
Four different environments were tested 100~times each for a 17-DoF kinematic chain, with decreasing redundancy: The Empty environment, a Random Easy environment, a more Cluttered Random environment, and the Horn environment; see Fig.~\ref{fig:res}. In all the cases, the RRT$^+$ versions of each planner were faster and in most cases significantly faster.  Details appear in Fig.~\ref{fig:graph} and Table~\ref{table}. 

As can be observed, the PrioritizedRRT$^+$-Connect not only outperformed the RRT-Connect by a wide margin, but also 
all the other planners using uniform sampling. Additionally, it outperformed robust planners with biased sampling in
environments with more redundancy, but KPIECE~\cite{csucan2009kinodynamic} and STRIDE~\cite{gipson2013resolution} are the clear winners in less redundant problems.

Interestingly, for each problem, the single fastest solution across all trials
was generated by PrioritizedRRT$^+$-Connect.  This suggests a good choice of
prioritization may give faster results in a consistent way.

\invis{Currently, we are studying on implementing the method on KUKA YouBot mobile manipulator of 10 DoF, shown in Figure \ref{fig:kuka}, 
by planning simultaneously for the manipulator and the omni-directional base. The base is modeled as 
two prismatic joints for the translation, and one revolute for the orientation.

\begin{figure}[ht]
\centering
\fbox{\includegraphics[width=0.4\textwidth, clip=true]{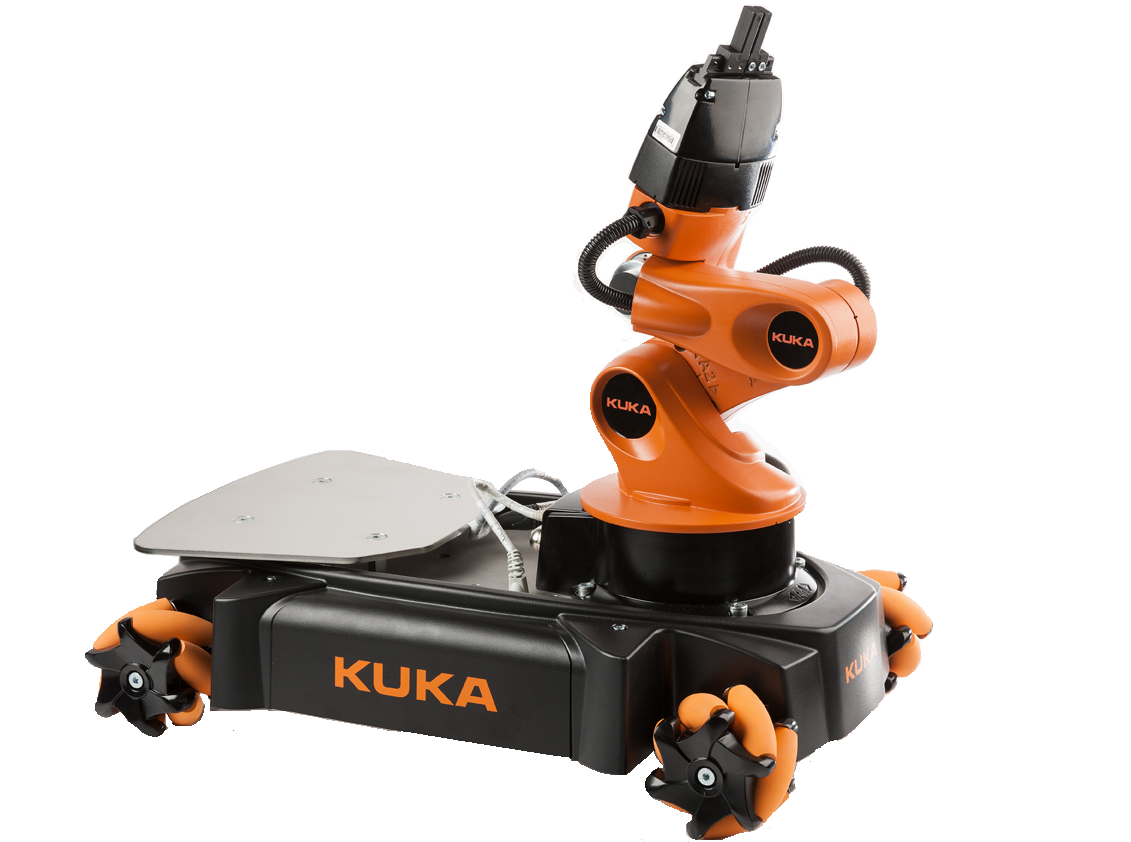}}
\caption{The KUKA YouBot ombi-directional mobile manipulator.}
\label{fig:kuka}
\end{figure}
}
\begin{figure*}[thpb]
 \begin{center}
  \leavevmode
   \begin{tabular}{cccc}
     \subfigure[]{\fbox{\includegraphics[width=0.2\textwidth]{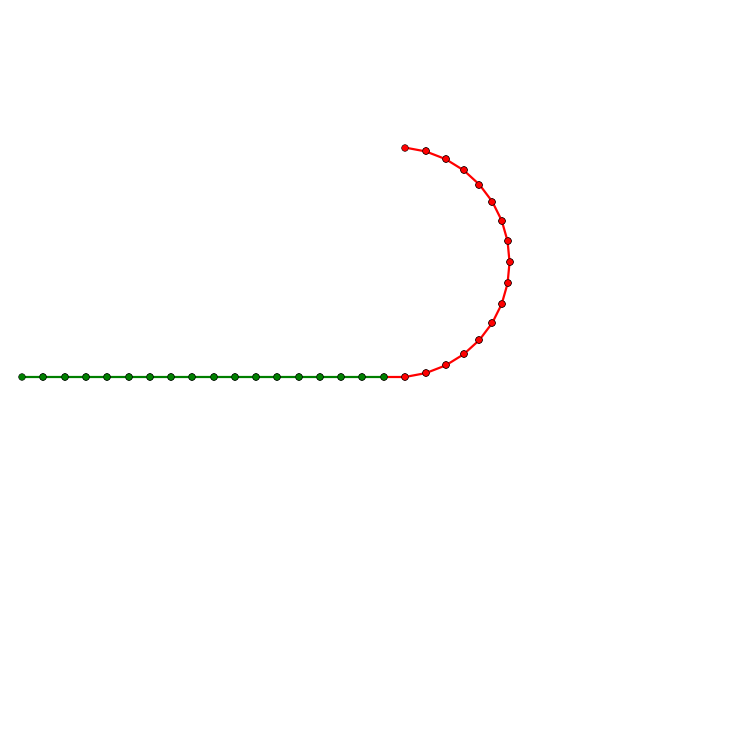}}}&
     \subfigure[]{\fbox{\includegraphics[width=0.201\textwidth]{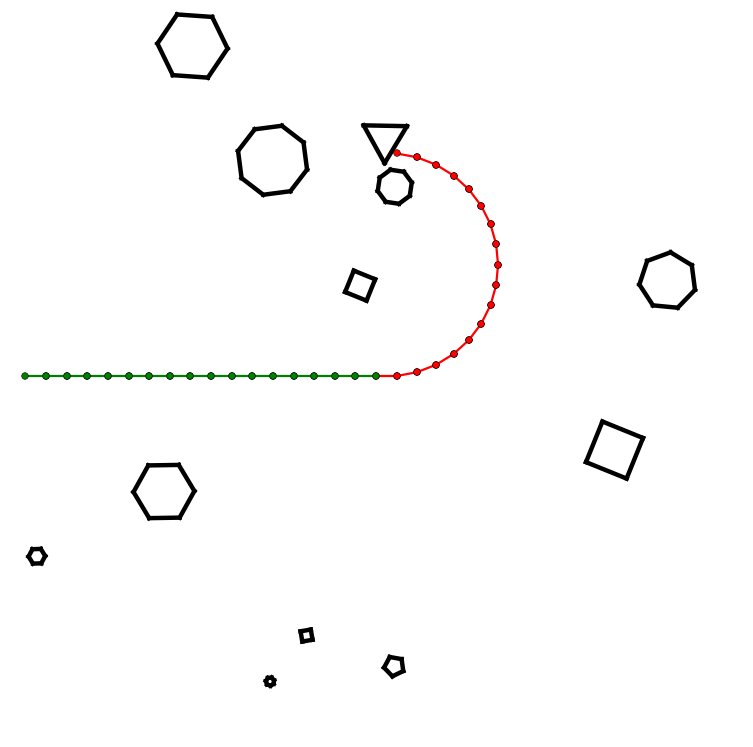}}}&
     \subfigure[]{\fbox{\includegraphics[width=0.197\textwidth]{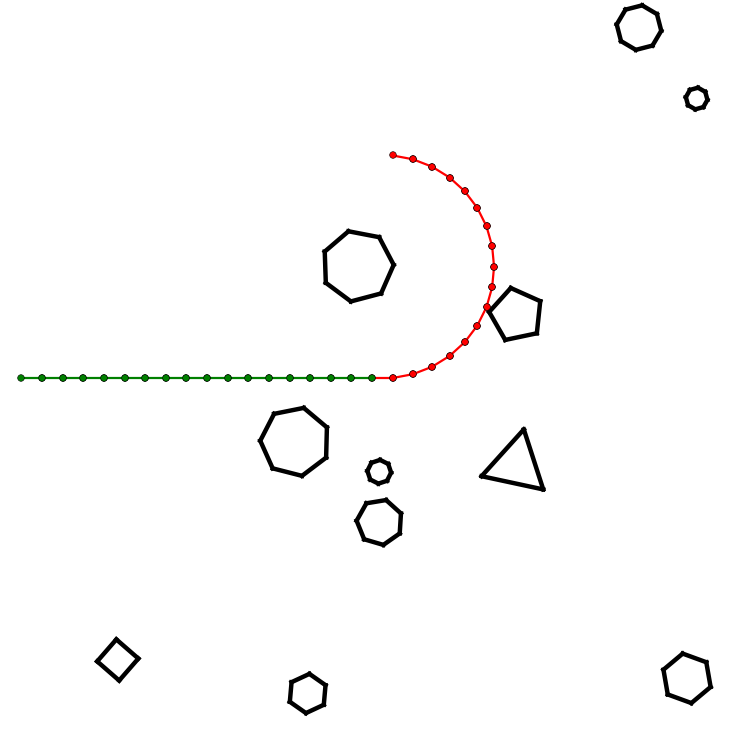}}}&
     \subfigure[]{\fbox{\includegraphics[width=0.205\textwidth]{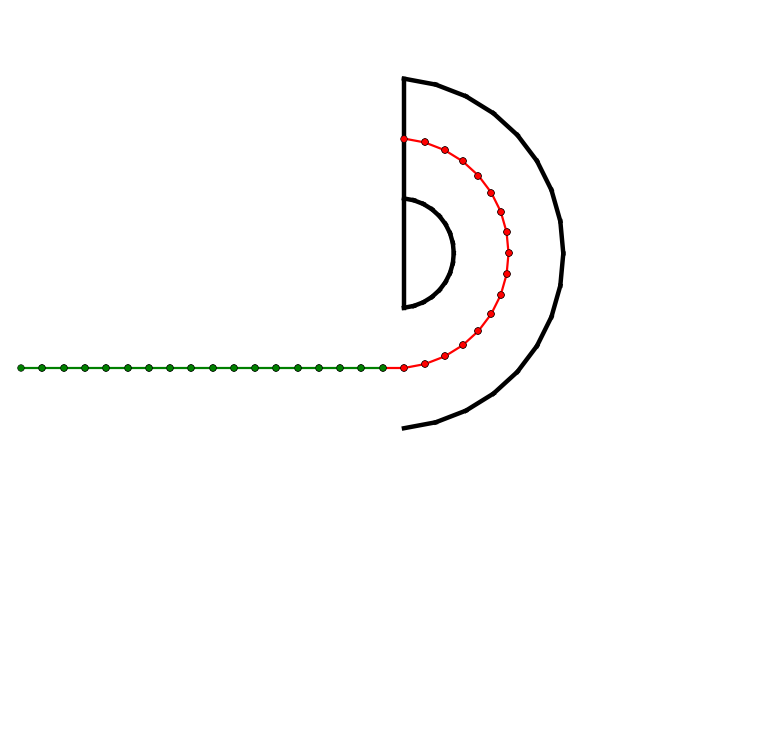}}}\\
     \subfigure[]{\fbox{\includegraphics[width=0.2\textwidth]{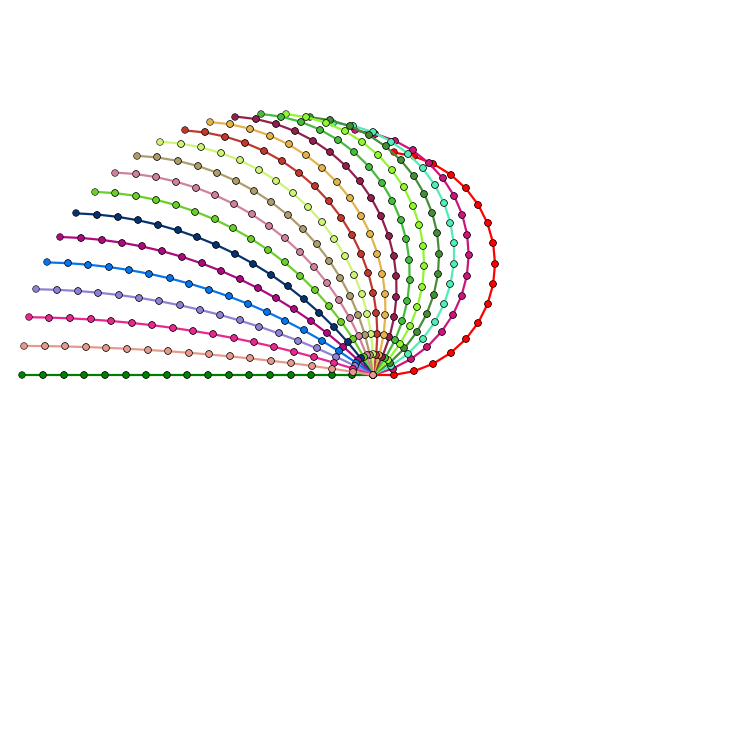}}}&
     \subfigure[]{\fbox{\includegraphics[width=0.2032\textwidth]{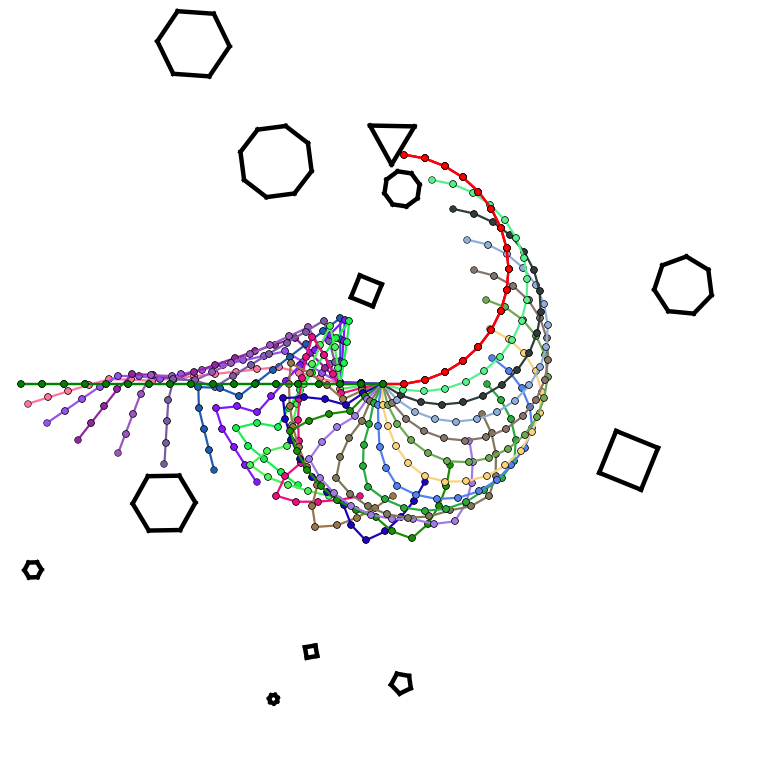}}}&
     \subfigure[]{\fbox{\includegraphics[width=0.1995\textwidth]{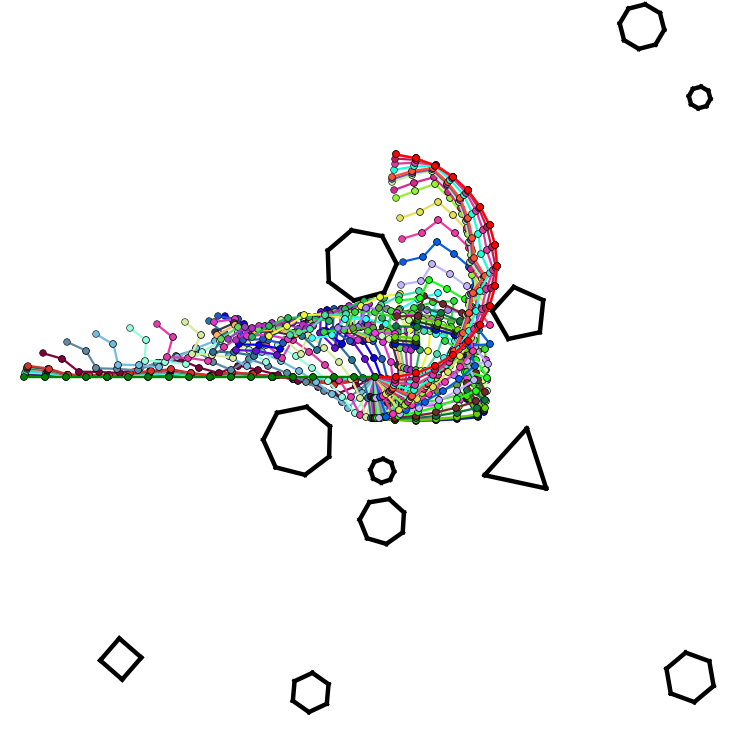}}}&
     \subfigure[]{\fbox{\includegraphics[width=0.204\textwidth]{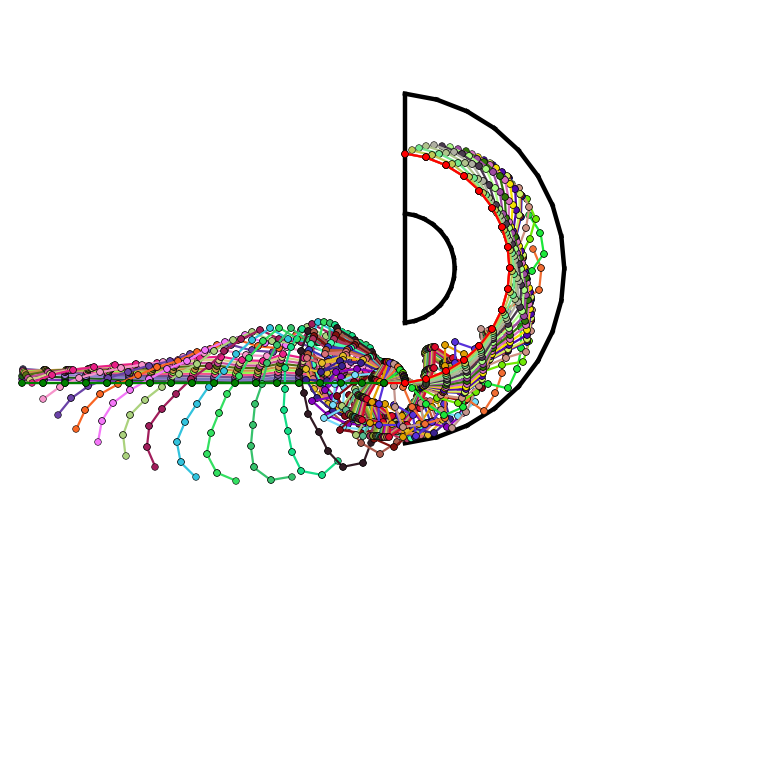}}} 
     \end{tabular}
  \end{center}
\caption{The four different environments the Empty, the Easy Random, the Cluttered Random and the Horn (a-d) with red indicating the $\qinit$ and green the $\qgoal$, and four different solutions for those environments (e-h) with PrioritizedRRT$^+$-Connect respectively.}
\label{fig:res}
\end{figure*}

\begin{figure*}[hpb]
 \begin{center}
  \leavevmode
   \begin{tabular}{cc}
     \subfigure[]{{\includegraphics[width=0.4\textwidth]{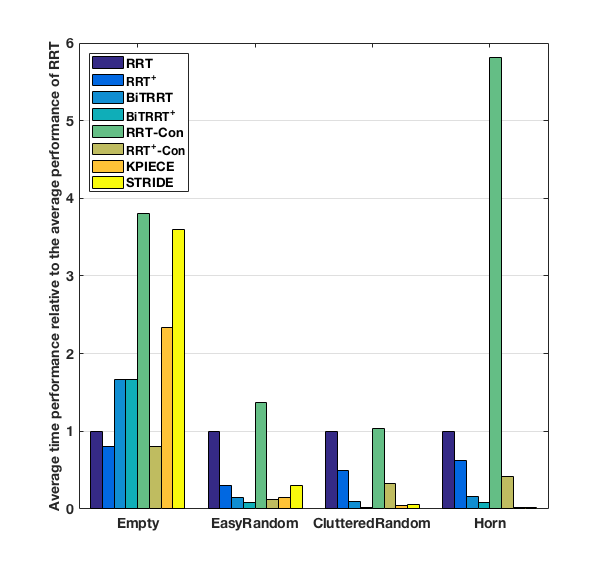}}\label{fig:mean}}&
     \subfigure[]{{\includegraphics[width=0.4\textwidth]{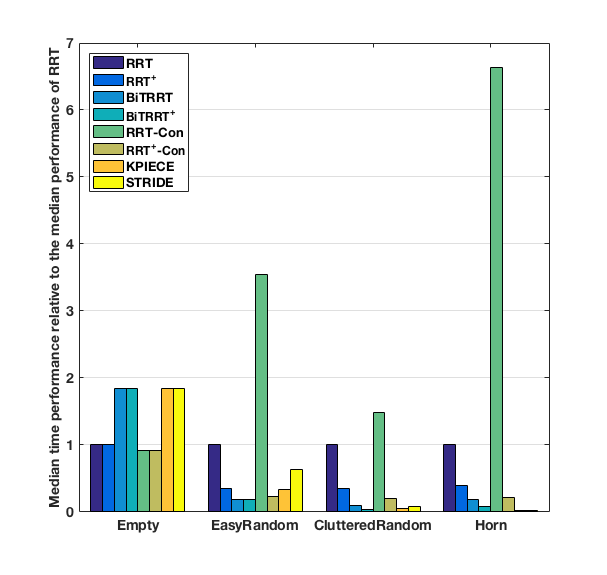}}\label{fig:median}}
	\end{tabular}
  \end{center}
\caption{The mean \subref{fig:mean} and the median \subref{fig:median} for the four environments and the eight planners scaled by the execution time of the RRT. Please refer to Fig. \ref{fig:res} for the four test\hyp environments and to Table \ref{table} for the numerical values of the mean and standard deviation performance of the eight planners. Note that $^+$ indicate the proposed algorithms.}
\label{fig:graph}
\end{figure*}

\begin{table*}[t]
\centering
\caption{Averages with standard deviation in the 4 environments and for 100 runs for every planner in seconds.}
\label{table}
\begin{tabular}{|l|l|l|l|l|}
\hline
             & Empty             & Easy Random       & Cluttered Random          & Horn               \\ \hline
RRT                       & 0.0015$\pm$ 0.0004 & 0.2134$\pm$ 0.3551 & 14.1118$\pm$ 16.8046 & 11.4807$\pm$ 12.2223 \\ \hline
RRT$^+$     			  & 0.0012$\pm$ 0.0002 & 0.0658$\pm$ 0.1268 & \text{ } 7.0245$\pm$ 11.3234 & \text{ } 7.2384$\pm$ 10.8390  \\ \hline
BiTRRT                    & 0.0025$\pm$ 0.0036 & 0.0306$\pm$ 0.0721 & \text{ } 1.4039$\pm$ 1.7033 & \text{ } 1.7877$\pm$ 2.2717   \\ \hline
BiTRRT$^+$				  & 0.0025$\pm$ 0.0009 & 0.0179$\pm$ 0.02   & \text{ } 0.3434$\pm$ 0.4004  & \text{ } 0.9814$\pm$ 1.5429   \\ \hline
RRT-Con                   & 0.0057$\pm$ 0.0173 & 0.2924$\pm$ 0.7734 & 14.5406$\pm$ 11.1380 & 66.7144$\pm$ 76.6271 \\ \hline
RRT$^+$-Con 			  & 0.0012$\pm$ 0.0001 & 0.0267$\pm$ 0.0282 & \text{ } 4.6150$\pm$ 12.4115 & \text{ } 4.7672$\pm$ 8.9909   \\ \hline
KPIECE                    & 0.0035$\pm$ 0.0026 & 0.0315$\pm$ 0.0223 & \text{ } 0.5744$\pm$ 0.51    & \text{ } 0.1537$\pm$ 0.1153   \\ \hline
STRIDE                    & 0.0054$\pm$ 0.0044 & 0.0656$\pm$ 0.0529 & \text{ } 0.8630$\pm$ 0.75    & \text{ } 0.1800$\pm$ 0.4681   \\ \hline
\end{tabular}
\end{table*}

\section{Conclusion} 
\label{sec:Conclusion}
We presented a general novel family of algorithms for fast motion planning in high dimensional configuration spaces. The algorithm should provide in the common case faster solutions than the regular RRT based methods, though in the uncommon worst case the solutions are somewhat slower.  \invis{The motivation of the work is that in many unpredictable situations that a robotic system can encounter, fast solutions can guarantee its survival while belated optimal solutions may cause severe damages.}

The algorithm is general enough to be applied to a broad variety of motion
planning problems.  For example, our experiments show potential for planning
with costmaps via adaptation of the bidirectional TRRT. 
The planner also can be adjusted to each problem and provide faster results
than the regular versions of the planners.

\invis{Additionally, RRT$^+$ arises new problems that are both related to applications
in complex robotic systems and the further development of the idea.}
A number of important question remain unanswered.  For example, it would be
very interesting to find an efficient way to pick, while searching, subspaces
that are more likely to contain solutions. Also, it is important to provide a
general way to identify when a new iteration should begin, using a metric of the
expansion of the tree, and eliminate the sensitivity to the $Q_{\rm max}$
parameter. Moreover, by using the previous metric it is possible to identify when
the tree overcame a difficult area and then reduce the dimensionality of the
search, in order to accelerate the results further.

Currently, there are two planners in OMPL that outperform the
PrioritizedRRT$^+$: KPIECE and STRIDE, which sample non-uniformly with a bias
to narrow spaces.  The integration of the ideas underlying RRT$^+$ into those
planners to accelerate the results is left for a future work.

Lastly, we plan to explore in the future the ability of the planner to
efficiently produce paths that satisfy some natural constraints of each system.

\section*{ACKNOWLEDGMENT}

The authors would like to thank the generous support of the Google Faculty Research Award and the National Science Foundation grants (NSF 0953503, 1513203, 1526862, 1637876).

\invis{Here is Jason's grant info. There should be two:

This material is based upon work supported by the National Science
Foundation under Grant No. 1526862.
This material is based upon work supported by the National Science
Foundation under Grant No. 0953503.}



\bibliography{ref}

\begin{thebibliography}{10}
\providecommand{\url}[1]{#1}
\csname url@rmstyle\endcsname
\providecommand{\newblock}{\relax}
\providecommand{\bibinfo}[2]{#2}
\providecommand\BIBentrySTDinterwordspacing{\spaceskip=0pt\relax}
\providecommand\BIBentryALTinterwordstretchfactor{4}
\providecommand\BIBentryALTinterwordspacing{\spaceskip=\fontdimen2\font plus
\BIBentryALTinterwordstretchfactor\fontdimen3\font minus
  \fontdimen4\font\relax}
\providecommand\BIBforeignlanguage[2]{{%
\expandafter\ifx\csname l@#1\endcsname\relax
\typeout{** WARNING: IEEEtran.bst: No hyphenation pattern has been}%
\typeout{** loaded for the language `#1'. Using the pattern for}%
\typeout{** the default language instead.}%
\else
\language=\csname l@#1\endcsname
\fi
#2}}

\bibitem{kuo1994mechanical}
A.~D. Kuo, ``A mechanical analysis of force distribution between redundant,
  multiple degree-of-freedom actuators in the human: Implications for the
  central nervous system,'' \emph{Human movement science}, vol.~13, no.~5, pp.
  635--663, 1994.

\bibitem{RekleitisMed2016}
M.~Xanthidis, K.~J. Kyriakopoulos, and I.~Rekleitis, ``Dynamically efficient
  kinematics for hyper-redundant manipulators,'' in \emph{The 24th
  Mediterranean Conf. on Control and Automation}, Athens, Greece, Jun. 2016,
  pp. 207--213.

\bibitem{reif1994motion}
J.~Reif and M.~Sharir, ``Motion planning in the presence of moving obstacles,''
  \emph{Journal of the ACM}, vol.~41, no.~4, pp. 764--790, 1994.

\bibitem{ma1994development}
S.~Ma, S.~Hirose, and H.~Yoshinada, ``Development of a hyper-redundant
  multijoint manipulator for maintenance of nuclear reactors,'' \emph{Advanced
  robotics}, vol.~9, no.~3, pp. 281--300, 1994.

\bibitem{liljeback2006snakefighter}
P.~Liljeback, O.~Stavdahl, and A.~Beitnes, ``Snakefighter-development of a
  water hydraulic fire fighting snake robot,'' in \emph{9th Int. Conf. on
  Control, Automation, Robotics and Vision, (ICARCV)}, 2006, pp. 1--6.

\bibitem{ikuta2003hyper}
K.~Ikuta, T.~Hasegawa, and S.~Daifu, ``Hyper redundant miniature manipulator"
  hyper finger" for remote minimally invasive surgery in deep area,'' in
  \emph{IEEE Int. Conf. on Robotics and Automation}, vol.~1, 2003, pp.
  1098--1102.

\bibitem{dubowsky1989planning}
S.~Dubowsky and E.~Vance, ``Planning mobile manipulator motions considering
  vehicle dynamic stability constraints,'' in \emph{IEEE Int. Conf. on Robotics
  and Automation}, 1989, pp. 1271--1276.

\bibitem{yamamoto1992coordinating}
Y.~Yamamoto and X.~Yun, ``Coordinating locomotion and manipulation of a mobile
  manipulator,'' in \emph{Proc. of the 31st IEEE Conf. on Decision and
  Control}, 1992, pp. 2643--2648.

\bibitem{khatib1996vehicle}
O.~Khatib, K.~Yokoi, K.~Chang, D.~Ruspini, R.~Holmberg, and A.~Casal,
  ``Vehicle/arm coordination and multiple mobile manipulator decentralized
  cooperation,'' in \emph{Proc. of the IEEE/RSJ Int. Conf. on Intelligent
  Robots and Systems}, vol.~2, 1996, pp. 546--553.

\bibitem{buckley1989fast}
S.~J. Buckley, ``Fast motion planning for multiple moving robots,'' in
  \emph{IEEE Int. Conf. on Robotics and Automation}, 1989, pp. 322--326.

\bibitem{cao1997cooperative}
Y.~U. Cao, A.~S. Fukunaga, and A.~Kahng, ``Cooperative mobile robotics:
  Antecedents and directions,'' \emph{Autonomous robots}, vol.~4, no.~1, pp.
  7--27, 1997.

\bibitem{lavalle2006planning}
S.~M. LaValle, \emph{{Planning Algorithms}}.\hskip 1em plus 0.5em minus
  0.4em\relax Cambridge Univ. press, 2006.

\bibitem{kavraki1996probabilistic}
L.~E. Kavraki, P.~{\v{S}}vestka, J.-C. Latombe, and M.~H. Overmars,
  ``Probabilistic roadmaps for path planning in high-dimensional configuration
  spaces,'' \emph{IEEE Trans. on Robotics and Automation}, vol.~12, no.~4, pp.
  566--580, 1996.

\bibitem{lavalle1998rapidly}
S.~M. LaValle, ``Rapidly-exploring random trees: A new tool for path
  planning,'' Computer Science Dept., Iowa State University, Tech. Rep. TR
  98-11, Oct. 1998.

\bibitem{park2011collision}
J.~J. Park, H.~S. Kim, and J.-B. Song, ``Collision-free path planning for a
  redundant manipulator based on prm and potential field methods,''
  \emph{Journal of Institute of Control, Robotics and Systems}, vol.~17, no.~4,
  pp. 362--367, 2011.

\bibitem{bertram2006integrated}
D.~Bertram, J.~Kuffner, R.~Dillmann, and T.~Asfour, ``An integrated approach to
  inverse kinematics and path planning for redundant manipulators,'' in
  \emph{Proc. IEEE Int. Conf. on Robotics and Automation}, 2006, pp.
  1874--1879.

\bibitem{weghe2007randomized}
M.~V. Weghe, D.~Ferguson, and S.~S. Srinivasa, ``Randomized path planning for
  redundant manipulators without inverse kinematics,'' in \emph{7th IEEE-RAS
  Int. Conf. on Humanoid Robots}, 2007, pp. 477--482.

\bibitem{qian2013path}
Y.~Qian and A.~Rahmani, ``Path planning approach for redundant manipulator
  based on jacobian pseudoinverse-rrt algorithm,'' in \emph{6th Int. Conf. on
  Intelligent Robotics and Applications}, Busan, South Korea, Sep. 2013, pp.
  706--717.

\bibitem{vannoy2008real}
J.~Vannoy and J.~Xiao, ``Real-time adaptive motion planning (ramp) of mobile
  manipulators in dynamic environments with unforeseen changes,'' \emph{IEEE
  Trans. on Robotics}, vol.~24, no.~5, pp. 1199--1212, 2008.

\bibitem{berenson2008optimization}
D.~Berenson, J.~Kuffner, and H.~Choset, ``An optimization approach to planning
  for mobile manipulation,'' in \emph{IEEE Int. Conf. on Robotics and
  Automation}, 2008, pp. 1187--1192.

\bibitem{van2005prioritized}
J.~P. Van Den~Berg and M.~H. Overmars, ``Prioritized motion planning for
  multiple robots,'' in \emph{IEEE/RSJ Int. Conf. on Intelligent Robots and
  Systems}, 2005, pp. 430--435.

\bibitem{carpin2002parallel}
S.~Carpin and E.~Pagello, ``On parallel rrts for multi-robot systems,'' in
  \emph{Proc. 8th Conf. Italian Association for Artificial Intelligence}, 2002,
  pp. 834--841.

\bibitem{wagner2015subdimensional}
G.~Wagner, ``Subdimensional expansion: A framework for computationally
  tractable multirobot path planning,'' \emph{Master thesis}, 2015.

\bibitem{otani2009applying}
T.~Otani and M.~Koshino, ``Applying a path planner based on rrt to cooperative
  multirobot box-pushing,'' \emph{Artificial Life and Robotics}, vol.~13,
  no.~2, pp. 418--422, 2009.

\bibitem{solovey2015finding}
K.~Solovey, O.~Salzman, and D.~Halperin, ``Finding a needle in an exponential
  haystack: Discrete rrt for exploration of implicit roadmaps in multi-robot
  motion planning,'' in \emph{Algorithmic Foundations of Robotics XI}.\hskip
  1em plus 0.5em minus 0.4em\relax Springer, 2015, pp. 591--607.

\bibitem{kuffner2002dynamically}
J.~J. Kuffner~Jr, S.~Kagami, K.~Nishiwaki, M.~Inaba, and H.~Inoue,
  ``Dynamically-stable motion planning for humanoid robots,'' \emph{Autonomous
  Robots}, vol.~12, no.~1, pp. 105--118, 2002.

\bibitem{kuffner2005motion}
J.~Kuffner, K.~Nishiwaki, S.~Kagami, M.~Inaba, and H.~Inoue, ``Motion planning
  for humanoid robots,'' in \emph{The Eleventh Int. Symposium Robotics
  Research}, 2005, pp. 365--374.

\bibitem{liu2012hierarchical}
H.~Liu, Q.~Sun, and T.~Zhang, ``Hierarchical rrt for humanoid robot footstep
  planning with multiple constraints in complex environments,'' in
  \emph{IEEE/RSJ Int. Conf. on Intelligent Robots and Systems}, 2012, pp.
  3187--3194.

\bibitem{vernaza2011efficient}
P.~Vernaza and D.~D. Lee, ``Efficient dynamic programming for high-dimensional,
  optimal motion planning by spectral learning of approximate value function
  symmetries,'' in \emph{IEEE Int. Conf. on Robotics and Automation}, 2011, pp.
  6121--6127.

\bibitem{yershova2005dynamic}
A.~Yershova, L.~Jaillet, T.~Sim{\'e}on, and S.~M. LaValle, ``{Dynamic-domain
  RRTs: Efficient exploration by controlling the sampling domain},'' in
  \emph{Proc. of the IEEE Int. Conf. on Robotics and Automation}, 2005, pp.
  3856--3861.

\bibitem{gipson2013resolution}
B.~Gipson, M.~Moll, and L.~E. Kavraki, ``Resolution independent density
  estimation for motion planning in high-dimensional spaces,'' in \emph{IEEE
  Int. Conf. on Robotics and Automation}, 2013, pp. 2437--2443.

\bibitem{gochev2011path}
K.~Gochev, B.~Cohen, J.~Butzke, A.~Safonova, and M.~Likhachev, ``Path planning
  with adaptive dimensionality,'' in \emph{Fourth annual symposium on
  combinatorial search}, 2011.

\bibitem{yoshida2005humanoid}
E.~Yoshida, ``Humanoid motion planning using multi-level dof exploitation based
  on randomized method,'' in \emph{2005 IEEE/RSJ Int. Conf. on Intelligent
  Robots and Systems}, 2005, pp. 3378--3383.

\bibitem{kim2015efficient}
D.-H. Kim, Y.-S. Choi, T.~Park, J.~Y. Lee, and C.-S. Han, ``Efficient path
  planning for high-dof articulated robots with adaptive dimensionality,'' in
  \emph{IEEE Int. Conf. on Robotics and Automation}, 2015, pp. 2355--2360.

\bibitem{shkolnik2009path}
A.~Shkolnik and R.~Tedrake, ``Path planning in 1000+ dimensions using a
  task-space voronoi bias,'' in \emph{IEEE International Conference on Robotics
  and Automation}.\hskip 1em plus 0.5em minus 0.4em\relax IEEE, 2009, pp.
  2061--2067.

\bibitem{bayazit2005iterative}
O.~B. Bayazit, D.~Xie, and N.~M. Amato, ``Iterative relaxation of constraints:
  A framework for improving automated motion planning,'' in \emph{IEEE/RSJ
  International Conference on Intelligent Robots and Systems}, 2005, pp.
  3433--3440.

\bibitem{kuffner2000rrt}
J.~J. Kuffner and S.~M. LaValle, ``Rrt-connect: An efficient approach to
  single-query path planning,'' in \emph{IEEE Int. Conf. on Robotics and
  Automation}, vol.~2, 2000, pp. 995--1001.

\bibitem{jaillet2008transition}
L.~Jaillet, J.~Cort{\'e}s, and T.~Sim{\'e}on, ``Transition-based rrt for path
  planning in continuous cost spaces,'' in \emph{IEEE/RSJ Int. Conf. on
  Intelligent Robots and Systems}, 2008, pp. 2145--2150.

\bibitem{linial1986hard}
N.~Linial, ``Hard enumeration problems in geometry and combinatorics,''
  \emph{SIAM Journal on Algebraic Discrete Methods}, vol.~7, no.~2, pp.
  331--335, 1986.

\bibitem{trikalinos2014efficient}
T.~Trikalinos and G.~van Valkenhoef, ``Efficient sampling from uniform density
  n-polytopes,'' 2014.

\bibitem{sucan2012open}
I.~A. Sucan, M.~Moll, and L.~E. Kavraki, ``The open motion planning library,''
  \emph{IEEE Robotics \& Automation Magazine}, vol.~19, no.~4, pp. 72--82,
  2012.

\bibitem{csucan2009kinodynamic}
I.~A. {\c{S}}ucan and L.~E. Kavraki, ``Kinodynamic motion planning by
  interior-exterior cell exploration,'' in \emph{Algorithmic Foundation of
  Robotics VIII}.\hskip 1em plus 0.5em minus 0.4em\relax Springer, 2009, pp.
  449--464.

\end{thebibliography}
\bibliographystyle{template/IEEEtran}

\end{document}